\def\year{2020}\relax
\documentclass[letterpaper]{article} 
\usepackage{aaai20}  
\usepackage{times}  
\usepackage{helvet} 
\usepackage{courier}  
\usepackage[hyphens]{url}  
\usepackage{graphicx} 
\usepackage{graphicx}  
\usepackage{amsmath,amssymb}
\usepackage{algorithm}
\usepackage{algorithmic}
\usepackage{amsthm}
\usepackage{booktabs,multirow}
\usepackage{graphicx}
\usepackage{subcaption}
\newtheorem{theorem}{Theorem}

\newtheorem{observation}{Observation}
\newtheorem{lemma}{Lemma}
\newtheorem{definition}{Definition}

\title{Protecting Geolocation Privacy of Photo Collections}
\author{
  Jinghan Yang, Ayan Chakrabarti, Yevgeniy Vorobeychik\\
 Computer Science \& Engineering, Washington University in St. Louis\\
  \{jinghan.yang, ayan, yvorobeychik\}@wustl.edu
}
 
\begin{document}

\maketitle

\begin{abstract}
People increasingly share personal information, including their
photos and photo collections, on social media.
This information, however, can compromise individual privacy,
particularly as social media platforms use it to infer detailed models
of user behavior, including tracking their location.
We consider the specific issue of location privacy as potentially
revealed by posting photo collections, which facilitate accurate
geolocation with the help of deep learning methods even in the absence
of geotags.
One means to limit associated inadvertent geolocation privacy
disclosure is by carefully pruning select photos from photo
collections before these are posted publicly.
We study this problem formally as a combinatorial
optimization problem in the context of geolocation prediction
facilitated by deep learning.
We first demonstrate the complexity both by showing that a natural
greedy algorithm can be arbitrarily bad and by proving that the problem is
NP-Hard.
We then exhibit an important tractable special case, as well as a more
general approach based on mixed-integer linear programming.
Through extensive experiments on real photo collections, we demonstrate
that our approaches are indeed highly effective at preserving geolocation privacy.
\end{abstract}

\newcommand{\citep}[1]{\citeauthor{#1}~\shortcite{#1}}

\section{Introduction}

\begin{figure}
  \centering
  \includegraphics[width=0.825\columnwidth]{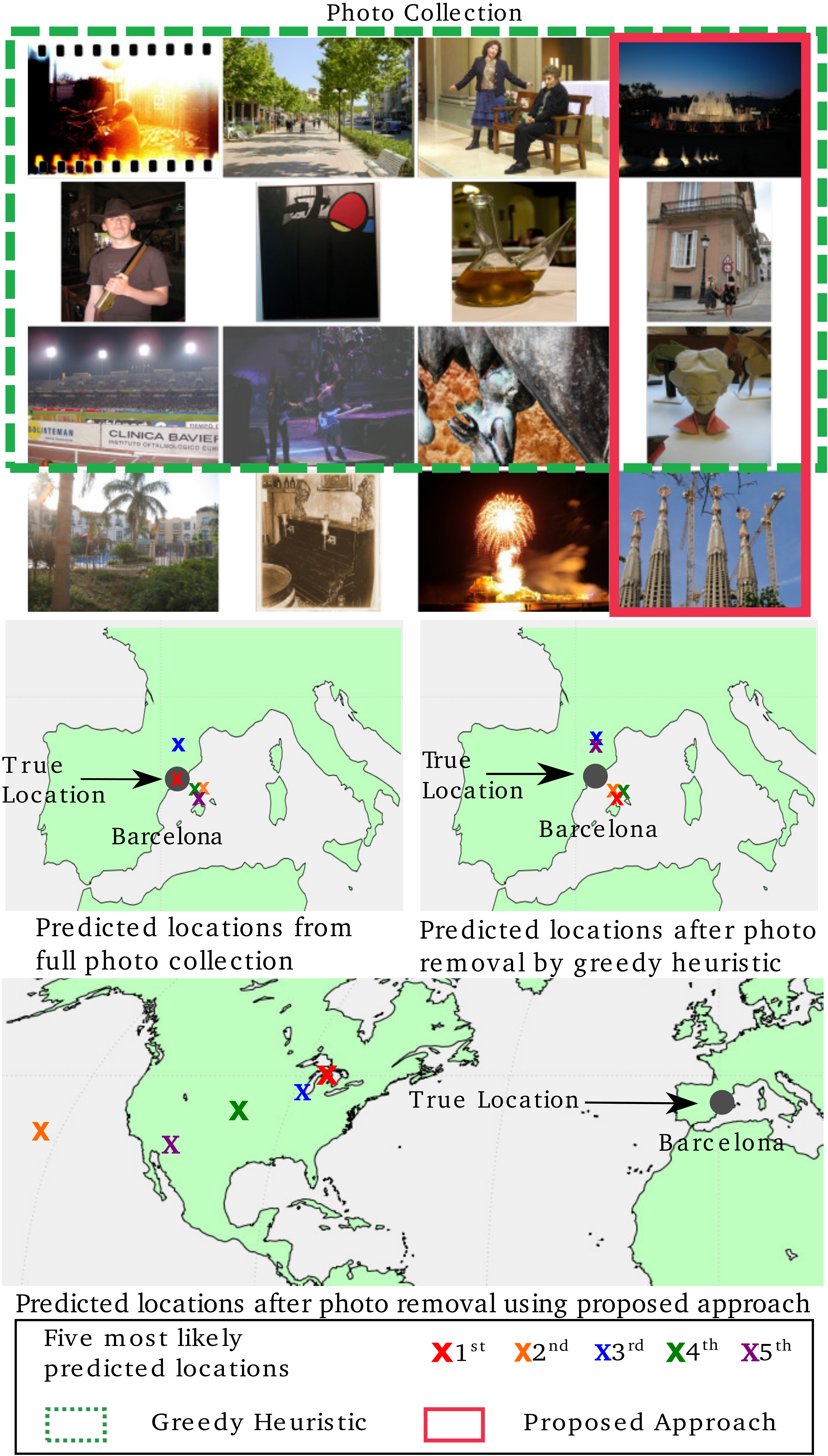}
  \caption{Illustration of our problem and approach.  This photo
    collection has 16 photos, all taken in Barcelona, which is also
    the correctly predicted location when all photos are used (top
    left map).
    Greedily removing photos to ensure that Barcelona is not among
    top-5 predicted locations corresponds to the selection highlighted
    in dashed green, and the resulting top-5 predictions are on the
    top right map.  Photos chosen by our approach are identified by a
    red box, and the result is shown on the lower map.
}
  \label{fig:teaser0}
\end{figure}

A person's physical whereabouts---the city, state, or country they are in or have traveled to---is sensitive information that many desire to keep private. With routine information about everyday activities increasingly being shared online and on social media, there is a heightened risk that such data may unintentionally leak information about the geographic location of individuals. For example, disclosure of a person's current location or travel history to law-enforcement agencies, hostile governments, employers, advertising companies, etc.\ could result in significant real harms. In this work, we specifically address the risk of geolocation privacy from collections of photographs shared by users online.

The IM2GPS method~\cite{Hays:2008:im2gps} was the first to demonstrate that, even without explicit embedded geotags, image content sometimes contains enough cues to allow accurate estimation of the location where the image was taken. The accuracy of image geolocation algorithms has since improved considerably, most recently through the use of deep convolutional neural networks (CNNs)~\cite{Weyand2016}. Nevertheless, geolocation prediction accuracy from individual images is still limited, and the content of a single photograph is typically insufficient to unambiguously resolve the location of the photographer.

However, as we show, geolocation predictions can be significantly more accurate when aggregating information across a \emph{collection} of photographs known to have been taken in the same geographic location (e.g., a person's photo-album from a specific trip, or images from different people known to be traveling together or attending the same event). While individual images can appear to have been taken in any of a large set of possible locations, the intersection of these sets for different images is often sufficient to narrow down the true location of the image collection.

We consider the problem of censoring photo collections in order to preserve geolocation privacy: selecting a minimal set of images to delete from a given collection, such that the true location is not included in the most likely locations predicted from the remaining images. We consider different variants of this task: (1) \emph{defeating top-1 accuracy} while minimizing deletions, to ensure the most likely predicted location is incorrect; (2) \emph{defeating top-$k$ accuracy}, such that the true location is not among the $k$ most likely predicted locations; and (3) \emph{maximizing the rank of the true location} in the list of likely predicted locations, under a fixed budget on the number of deleted images. We consider both \emph{white-box} and \emph{black-box} versions of these tasks.
In the white-box version, the privacy-preserving algorithm has access to the true location classifier.
In the black-box version, in contrast, the algorithm instead uses a proxy classifier.

We show that, except for the top-1 variant, this problem is NP-hard, because the optimal deletion set should not just censor the images that are most indicative of the true location, but also maintain the likelihood of plausible alternatives. We also specifically show that a naive greedy algorithm can be arbitrarily sub-optimal. Then, we propose a polynomial-time algorithm for the top-1 variant and present a mixed-integer linear programming (MILP) formulation for the other variants to allow the use of standard MILP solvers. 

Finally, we present an extensive experimental evaluation of our proposed approaches using images and albums downloaded from Flickr. \footnote{https://www.flickr.com}
In particular, our experiments show that our methods significantly outperform simple greedy deletions in all task variants, and in many cases admit optimal solutions with only a small number of deleted images. Our work thus highlights the risk to geolocation privacy that exists from photo collections commonly shared by users today, while also providing an effective solution that comes at the relatively low cost of censoring a small fraction of images. Our code could be found in the project page \footnote{https://github.com/jinghanY/geoPrivacyAlbum}.

\noindent{\bf Related Work} 
Sharing photographs with embedded location information can allow images to tell a richer story of the places and events they represent. Consequently, many cameras (especially those on mobile phones) allow users the option of embedding GPS coordinates in image EXIF tags.

\citep{Hays:2008:im2gps} were the first to exploit such data to train geolocation predictors, and show that geographic localization is possible from image information alone. Since then, a number of methods have made steady progress on improving the accuracy of image location estimators through the use of deep convolutional neural networks (CNN). \citep{Weyand2016} formulate geolocation as a classification problem, by dividing the globe into a discrete number of cells and assigning a separate label for each. \citep{Howard:2017} demonstrated versions of these classifiers that can run on mobile devices, while \citep{kim2017crn} considered learning representations that exploit contextual information. \citep{Vo:2017} presented a study confirming the efficacy of \citeauthor{Weyand2016}, as well as of using deep network features for localization by retrieval. \citep{Tian2017CrossViewIM} devised a matching-based geolocation approach specifically designed for street-view images.

\citep{Ryan:2017} and \citep{Hal:2013} discuss the dangers of using
geolocation for commercial purposes and the importance of geolocation
privacy. In the security domain, most existing works propose ways to
anonymize location data by replacing the associated name with an
untraceable
ID~\cite{Beresford:2003:LPP:642730.642740,Gruteser:2003:AUL:1066116.1189037}.
Other work addresses geolocation privacy of user
interfaces~\cite{Doty:2010:GPA:1868470.1868485}. \citep{8201d37b4cd54e9b91a390f720178ccd},
\citep{Xiao:2015:PLD:2810103.2813640} and
\citep{Ho:2011:DPL:2071880.2071884} develop mechanisms for preserving
geolocation privacy using differential privacy techniques which add
noise to actual location information; as such, they do not deal with
the issue of inferring location from a photo collection, which is the focus of
our work.

However, these efforts do not address the issue of location inference based on published images and photo collections.
Our work is also broadly related to adversarial machine learning~\cite{Vorobeychik:2018}, although the specific problem and techniques differ considerably. In particular, our focus is not to exploit divergences between learned classifiers and human perception (by creating imperceptible perturbations that fool classifiers~\cite{szegedy2013intriguing,nguyen2015deep}), but rather to censor a small number of images to remove information that can be leveraged for geolocation.

\section{Problem Formulation}
\label{sec:Model}

\subsection{Geolocation Prediction from Images}
Predicting geographic location from a single image is often naturally
cast as a classification task~\cite{Weyand2016}. Possible locations on
the globe are discretized into a discrete set $\mathcal{C} =
[c_1,c_2,\ldots c_M]$ of $M$ geospatial regions, and a classifier is
trained to predict a probability distribution $P(c|X)$ from an input
image $X$ for $c\in\mathcal{C}$. Here, $P(c|X)$ denotes the probability or classifier's belief that the image $X$ was taken in the discretized location $c$.

While the accuracy of per-image location classifiers is still limited,
we consider the effect of aggregating location information across
images. Formally, we consider the task of estimating location from a
\emph{collection} $\mathcal{A} = [X_1, X_2, \ldots X_N]$ of $N$ images
that are known to have all been taken at the same location
$c$. Without making any assumption on the relationship between images
in $\mathcal{A}$ (i.e., treating them as conditionally independent
given $c$), and further assuming that all locations $c\in\mathcal{C}$ are equally likely \emph{a-priori}, the probability distribution of the collection $\mathcal{A}$ is simply given by the product of the individual image distributions:
\begin{equation}
  \label{eq:probmult}
  P(c|\mathcal{A}) = \frac{\prod_{i=1}^N P(c|X_i)}{\sum_{c'\in\mathcal{C}}\prod_{i=1}^N P(c'|X_i)}\cdot
\end{equation}

For convenience, we define $M-$dimensional score vectors $\{S_i\}$
as the log probabilities of locations from individual images, with the
$j$th element defined by $S_i^j = \log P(c_j | X_i)$.
Similarly, we define $S = \sum_{i=1}^N S_i$
as the score corresponding
to the image collection. 
This implies that $S^j = \log P(c_j|\mathcal{A}) + z$,
where 
$z$ is a \emph{scalar} corresponding to the normalization constant in \eqref{eq:probmult} (i.e., it is added to all elements of the vector $S$). Note that since $P(c_{j}|\mathcal{A}) > P(c_{j'}|\mathcal{A}) \Leftrightarrow S^j > S^{j'}$, a location $c^j$ has the same relative rank in both the probability and score vectors.

Thus, the location scores for an image collection are given simply by summing score vectors, as returned by a classifier, from individual images. As our experiments will show, aggregating information in this way from even a relatively small number of photos leads to significant gains in location accuracy compared to single images.

\subsection{Location Privacy by Selecting Images for Deletion}
\newcommand{\pparagraph}{\paragraph}

When using the distributions $P(c|\cdot)$ to make inferences about location, one can choose to look at either the most likely location---i.e., $\arg \max_{c} P(c|\cdot)$---or the set of $k$ most likely locations. Specifically, let $\mathcal{H}(c_j,S)$ define the set of locations that have higher scores in $S$ than a given $c_j$:
\begin{equation}
  \label{eq:hdef}
  \mathcal{H}(c_j,S) = \{j': c_{j'}\in\mathcal{C}, S^{j'} \ge S^j, j \neq j'\}.
\end{equation}
If the true location for a collection is $c_{{{t}}}$, then our location estimate is exactly correct (top-1 accurate) if $|\mathcal{H}(c_{{{t}}},S)| = 0$, and top-$k$ accurate if $|\mathcal{H}(c_{{{t}}},S)| < k$.

Our goal is to develop a privacy protection method that, given an
image set $\mathcal{A}$ and per-image scores $\{S_i\}$, selects a
small set of images $\mathcal{D} \subset \mathcal{A}$ to delete from
$\mathcal{A}$ to yield a censored image set $\mathcal{A'} =
\mathcal{A}\setminus\mathcal{D}$, so that score vector $S'$ from the
remaining images, i.e., $S' = \sum_{i:X_i\in\mathcal{A}'} S_i$, leads
to incorrect location estimates.
Next, we define different variants of this task for different objectives trading-off utility (number of deletions) and privacy.

\noindent{\bf Top-$k$ Guarantee:} The first variant seeks to minimize the number of deletions, while ensuring that the resulting scores from the censored image set push the true location $c_{{{t}}}$ out of the top-$k$ location results, for a specified $k$, i.e.,
\begin{equation}
  \label{eq:topkmind}
  \min_{\mathcal{D}} |\mathcal{D}|, \text{~~s.t.~~} |\mathcal{H}(c_{{{t}}},S')| \geq k.
\end{equation}
A special case of this is when $k=1$, when we wish to
ensure that the true location isn't predicted to be the most likely.

\noindent{\bf Fixed Deletion Budget:} Another variant works with a fixed budget $d$ on the number of images to delete, while trying to maximize the rank of the true label in the resulting score vector $S'$, i.e.,
\begin{equation}
  \label{eq:maxk}
  \max_{\mathcal{D}} |\mathcal{H}(c_{{{t}}},S')|, \text{~~s.t.~~} |\mathcal{D}| \le d.
\end{equation}
Here, we call the resulting value of $|\mathcal{H}(c_{{{t}}},S')|$ for the optimized deletion set as the ``protected-$k$'', since it defeats top-$k$ accuracy for all values of $k$ up to and including $|\mathcal{H}(c_{{{t}}},S')|$.

\noindent{\bf Incorporating User Preference:} In some cases, a user may want to specify a set of images $\mathcal{U}\subset \mathcal{A}$ that \emph{should not} be deleted. In this case, we add the constraint $\mathcal{D} \subset \mathcal{A}\setminus\mathcal{U}$ to the formulations \eqref{eq:topkmind} and \eqref{eq:maxk} above.

\noindent{\bf Black-box Protection:} So far, we have assumed that we have access to the location classifier that will be used for inference, and therefore to the corresponding scores $\{S_i\}$. However, in many real-world settings, our privacy protection algorithm will not have access to the actual classifier. In such cases, we propose using scores $\{S_i\}$ from a proxy classifier, and adding a scalar margin $\theta > 0$ to the scores $S_i^t$ of the true location in each image. These modified scores are used to solve the optimization problems \eqref{eq:topkmind} and \eqref{eq:maxk}, and
capture uncertainty about the true value of the scores $S_i^j$ (we can set $\theta$ such that the score advantage $S_i^t-S_i^j$, of the true location over any other location in a given image from our proxy, is likely to be within $\theta$ of the true advantage).

\section{Computational Hardness}

It is tempting to view the task of selecting images for deletion as computationally easy---that it can be solved simply with a greedy selection of images with highest probabilities or scores for the true location. However, this is not the case since optimal deletion must both reduce the cumulative score of the true location, \emph{and maintain high scores for possible alternate locations}. Deleting an image with high scores for other locations, as well as the true one, can often lead to no change, or even a decrease, in the relative rank of the true location. Conversely, it may sometimes be optimal to delete an image with a moderate true location score, when it is the only one in the collection with a very low score for an otherwise plausible alternate location. In this section, we discuss the sub-optimality of greedy selection, and then prove that optimal deletion is NP-hard.

Consider the greedy selection method that constructs a deletion set $\mathcal{D}$ by selecting images with the highest values of the score $S_i^t$ for the true location $c_t$. To defeat top-$k$ accuracy in \eqref{eq:topkmind} for a given $k$, it sequentially deletes images with decreasing values  $S_i^t$, stopping when $c_t$ is no longer in the top-$k$. Given a fixed budget $d$ for deletions in \eqref{eq:maxk}, it simply deletes the $d$ images with the highest scores $S_i^t$.
\begin{observation}
  Greedy selection can be arbitrarily sub-optimal. Specifically, in cases with just three possible discretized locations (i.e., $|\mathcal{C}|=3$), there exists a set of  score vectors corresponding to log-probabilities for $N+3$ images, $\forall N\geq 1$, where the first $N$ deletions by greedy selection leave the true location as the most likely prediction, whereas deleting two optimally selected images is sufficient to make it the least likely prediction (i.e., defeat top-$2$ accuracy).
  \label{thm:heuristic}
\end{observation}
Broadly, the above set corresponds to $N$ score vectors with uniform scores for all locations ($S_i^j=\log (1/3), \forall j$), and three additional vectors with lower than uniform true-location scores ($S_i^t < \log (1/3)$). From the latter, two vectors each feature a very low score for each of the incorrect locations, respectively. The greedy method selects the images with uniform distributions leading to no change in rank of the true location, while deleting the latter two images makes the true location the least likely prediction. The supplement provides the explicit construction of these score vectors and a complete proof.

Beyond worst-case sub-optimality, our experiments will show that greedy selection also yields poorer results than optimal deletion with our method in practice, with typical score vectors from real images. Figure~\ref{fig:teaser0} shows an example set of sixteen real images where, to achieve a top-$5$ location privacy guarantee, greedy selection requires twelve deletions while the optimal solution requires only four.

We next show that the optimization problems in our privacy formulation are in fact NP-hard.

\begin{theorem}
  The optimization problem in \eqref{eq:topkmind}, to minimize deletions for a top-$k$ guarantee, is NP-hard $\forall k\geq 2$.
  \label{thm:topkmind}
\end{theorem}

\begin{theorem}
The optimization problem in \eqref{eq:maxk}, for maximizing protected-$k$ under a fixed deletion budget, is NP-hard.
  \label{thm:maxk}
\end{theorem}

We begin by introducing and proving a lemma that relates to the NP-hardness of optimal deletion for making the true location less likely than a given set of alternate locations. We then use this lemma to prove Theorem \ref{thm:topkmind}, and subsequently, prove Theorem \ref{thm:maxk}.

\begin{lemma}
  Given an image collection $\mathcal{A}$ with true label $c_t$, and a target set $\bar{\mathcal{H}} \subset \mathcal{C}, c_t \notin |\bar{\mathcal{H}}|$  of alternate locations, with fixed size $k=|\bar{\mathcal{H}}|$, it is NP-hard $\forall k\geq 2$ to find the smallest deletion set $\mathcal{D} \subset \mathcal{A}~$, s.t. ${S'}^t \leq {S'}^j, \forall c_j \in \bar{\mathcal{H}}$, where $S'$ are score vectors for $\mathcal{A}'=\mathcal{A}\setminus\mathcal{D}$.
  \label{thm:fixed}
\end{lemma}
\begin{definition}Deletion for Two Alternate Locations (DTAL).
{\bf Given}: a collection $\mathcal{A}$, pair of alternate labels $\bar{\mathcal{H}}=\{c_p,c_q\}$, and true label $c_t\notin\bar{\mathcal{H}}$. {\bf Question}: is there any deletion set $\mathcal{D} \subset \mathcal{A}~$, s.t. $~{S'}^t \leq {S'}^j, \forall c_j \in \bar{\mathcal{H}}$ ?
\end{definition}
\noindent Note that although scores correspond to log-probabilities in our formulation, we place no restrictions on $S_i^j$ above, since given arbitrary score values, we can subtract an image-specific constant $\log \sum_j \exp(S_i^j)$ from the scores for \emph{all} labels to yield valid log-probabilities. This results in an identical decision problem, as the same constant occurs on both sides of all inequality constraints.

\begin{definition}Knapsack Problem (KP).
{\bf Given}: Values $\{v_i > 0\}$ and weights $\{w_i > 0\}$ for a set of objects $\mathcal{O}$, with a capacity constraint $W$ and value bound $V$. {\bf Question}: Is there a subset of objects $\mathcal{Q} \subseteq \mathcal{O}$ with total weight at most $W > 0$ and total value at least $V > 0$ ?
\end{definition}

\begin{proof}[Proof of Lemma \ref{thm:fixed}]
  It is easy to see DTAL is in NP. We prove DTAL is NP-complete by reduction from KP. Given an instance of KP, we form an instance of DTAL by creating a collection $\mathcal{A}$ with $|\mathcal{O}|+1$ images. We set $S_i^t=0$ for all images. For $1\leq i \leq |\mathcal{O}|$, we set the scores $S_i^p=v_i$ and $S_i^q=-w_i$. For the final image with $i=|\mathcal{O}+1|$, we set $S_i^p=-V$ and  $S_i^q=W$. Any deletion set returned by DTAL can not include the last image $X_{|\mathcal{O}|+1}$, since $S_i^t > S_i^q$ for all other images. Given this solution, we can get a corresponding solution for KP as the set of all objects for which the corresponding image $X_i$ was not included in the deletion set. Similarly, a solution of the KP instance corresponds to a solution of DTAL, where the deletion set includes images corresponding to all objects not included in the KP solution.

  Since DTAL---which asks if a deletion set, without any constraint on its size---is NP-complete, it follows that the task of determining if a deletion set of bounded size exists is also NP-complete. The latter corresponds to the decision version of the optimization problem in Lemma \ref{thm:fixed} for $k=2$. Therefore, the optimization in Lemma \ref{thm:fixed} is NP-hard for $k=2$. It is then easy to show the problem is NP-hard also for any $k>2$, since if there were a polynomial-time solver for $k>2$, we could call it to solve the $k=2$ version by setting the scores for all but two locations to $-\infty$, in all images.
\end{proof}

\begin{proof}[Proof of Theorem \ref{thm:topkmind}]
The optimization problem in Theorem \ref{thm:topkmind} requires making a similar decision as in Lemma \ref{thm:fixed}, except that it also requires determining the best set of alternate labels $\bar{\mathcal{H}}$ where $|\bar{\mathcal{H}}|=k, c_t \notin \bar{\mathcal{H}}$. It is easy to see that this is at least as hard as the fixed alternate target set variant. Specifically, if we have a solver for the problem in Theorem \ref{thm:topkmind}, we can use it to solve the problem in Lemma \ref{thm:fixed} by simply setting the scores $S_i^j = \infty, \forall i, \forall j: j \neq c_t, j \notin \bar{\mathcal{H}}$.
\end{proof}

\begin{proof}[Proof of Theorem \ref{thm:maxk}]
The fixed deletion budget optimization problem in Theorem \ref{thm:maxk} is at least as hard as the top-$k$ guarantee problem in  Theorem \ref{thm:topkmind}. If we have a solver for the fixed budget problem, we can call that solver sequentially with deletion budget $d = (1, 2,\ldots N)$, and stop whenever the returned protected-$k$ satisfies the desired top-$k$ guarantee.
\end{proof}

\section{Proposed Approach}

We now present algorithms for solving the privacy-preserving optimization problems in our formulation. We begin with a generic approach that maps our objectives and constraints to a mixed integer linear program (MILP). This formulation is naturally also NP-hard,  but it allows us to use standard MILP solvers that use various heuristics to find a solution efficiently. Finally, we had shown that the top-k guarantee problem is NP-hard $\forall k\geq 2$. For the important special case $k=1$, we show that this problem can in fact be solved in polynomial time.t

\subsection{Mixed Integer Linear Programming}
Given a collection $\mathcal{A}$ with set of locations $\mathcal{C}$, true location $c_t$, and score vectors $\{S_i\}$, we let $z_i\in\{0,1\}$ denote whether a photo is deleted ($z_i = 0$) or kept ($z_i = 1$).  
We let $h_j \in \{0,1\}, \forall j\neq t$ denote whether location $c_j$ has a lower ($h_j=0$) or higher ($h_j=1$) score ${S'}^j$ than that of the true location ${S'}^t$ in the censored collection $\mathcal{A}'$. We enforce consistency of these values of $h_j$ to the assignment $z_i$ by adding the constraints:
\begin{equation}
  \label{eq:hjdef}
  h_j \times \left( \sum_i z_i \left( S_i^j - S_t^t \right) \right) \geq 0,~~\forall j \neq t.
\end{equation}

Under these definitions, the protected-$k$ of a given assignment is given by $\sum_{j\neq t} h_j$, and the number of deletions by $\sum_i (1-z_i)$. 
Thus, our problem variant \eqref{eq:topkmind} where we minimize the number of deleted photos that ensure top-$k$ location privacy corresponds to solving
\begin{equation}
  \label{eq:topkilp}
  \max_{\{z_i\},\{h_j\}} \sum_i z_i,\qquad \text{s.t~~} \sum_{j\neq t}h_j \geq k,
\end{equation}
since minimizing the number of deleted photos is equivalent to maximizing the number of photos remaining in the censored collection.
Analogously, the problem of maximizing protected-$k$ with a fixed deletion budget $d$ corresponds to:
\begin{equation}
  \label{eq:maxkilp}
  \max_{\{z_i\},\{h_j\}} \sum_{j\neq t}h_j, \qquad \text{s.t~~} \sum_i (1-z_i) \leq d.
\end{equation}
Both problems above are over the binary variables $z_i\in\{0,1\}, h_j\in\{0,1\}$, under the additional constraints in \eqref{eq:hjdef}. 

Note that the formulations above involve bilinear constraints \eqref{eq:hjdef}, as these feature products of $h_j$ and $z_i$.
However, since these are both binary variables, it is straightforward to linearize the constraints using McCormick inequalities (see the supplement for details).

\noindent\textbf{User Preference:} User specified constraints that a set of images $\mathcal{U}$ not be deleted can be incorporated simply by adding the constraints $z_i=1, ~~~\forall X_i\in\mathcal{U}$ to the MILP.

\noindent\textbf{Black-Box Protection:} The margin $\theta$ described in Sec.~\ref{sec:Model}, that accounts for discrepancies between the proxy and un-known true classifiers, can also be incorporated in our MILP by simply modifying the constraints in \eqref{eq:hjdef} above to:
\begin{equation}
  \label{eq:hjdef2}
  h_j \times \left( \sum_i z_i \left( S_i^j - S_t^t - \theta \right) \right) \geq 0,~~\forall j \neq t.
\end{equation}

\subsection{Polynomial-Time Algorithm for top-1 Guarantee}

As we've shown, the optimization problems in our formulation are NP-hard for $k\geq 2$. Conversely, the special case of minimizing deletions for top-1 accuracy admits a solution in polynomial time. We begin by noting that, for this case, we must only ensure that the true location $c_t$ has a lower score than \emph{any} other single location $c_j, j\neq t$. 

First, note that for a specific choice of $c_j$, the assignment $\bar{z}^j_i \in \{0,1\}$ with fewest deletions that ensures $c_j$ has a higher score than $c_t$ is given by
\begin{equation}
  \label{eq:top1m}
  \{\bar{z}^j_i\} = \arg \max_{\{z_i\}} \sum_i z_i,{~~s.t.~~} \sum_i z_i (S^j_i - S^t_i) \geq 0,
\end{equation}
where $\bar{z}^j_i=0$ indicates that $X_i$ is deleted, as previously.
This can be solved by sorting the images $X_i$ in descending order of the score advantage $(S^t_i-S^j_i)$ of the true location $c_t$ over $c_j$, and deleting images (setting $\bar{z}^j_i=0$) sequentially till the constraint above is satisfied. The solution to the original problem is then given by searching over all possible $c_j \neq c_t$, i.e., $\min_{j\neq t} (\sum_i 1-\bar{z}^j_i)$.
Thus, we solve \eqref{eq:top1m} for each possible alternate location $c_j\neq c_t$, compare the number of deletions required for each, and select the answer with the fewest deletions. The overall algorithm runs in $O(MN\log N)$ time (where $N=|\mathcal{A}|$ is the number of images, and $M=|\mathcal{C}|$ the number of locations). The full algorithm is included in the supplement.

\section{Experiments}

We now present an extensive evaluation of our method for all variants of our geolocation privacy problem, measuring their success at preserving privacy against a state-of-the-art deep neural network-based location classifier with experiments on real photographs downloaded from Flickr.

\subsection{Preliminaries}

\noindent{\bf Datasets and Discretized Locations:} We conduct experiments on a set of 750k images downloaded from Flickr, where we ensure each image has a geotag to specify ground-truth location. We randomly split these into training and test sets, of size $90\%$ and $10\%$ of total images respectively. Following \citep{Weyand2016}, we cast geolocation as a multi-class classification problem over a set of discretized locations $\mathcal{C}$. This set is formed by an adaptive scheme to partition the globe into contiguous cells, such that each cell contains the same number of training images. This means that cell sizes vary depending on the density of photos (and indirectly, of population) in the corresponding region (e.g., oceans tend to inhabit large cells, while large metropolitan areas may be partitioned among multiple cells). We use a discretization into 512 cells (i.e., $|\mathcal{C}|=512$).

\noindent{\bf Photo Collections:} Photo collections for evaluation are constructed by randomly forming sets of images with the same discretized location from the test set. We consider collections of size 16, 32, 64, and 128, where collections for each size are formed as 10 random partitions of the entire test-set. This gives us roughly 45k collections of size 16, 21k of size 32, 10k of size 16, and 5k of size 128. Further, we collected 10k Flickr user designated ``albums'', each with images all from the same discretized location, and with no image overlap with the training set. The album sizes ranged between 16 and 250 (33 photos on average), with all 512 locations represented. In our evaluations, we treat these albums as photo collections, although images in photo albums tend to be more correlated than random collections.

\noindent{\bf Location Classifier:} We use a deep convolutional neural network---with a standard VGG-16~\cite{Vo:2017} architecture---to produce the probability distribution over locations from a single input image (the last layer of the network is modified from the standard architecture to have 512 channels, to match the size of $\mathcal{C}$). We initialize the weights of all but the last layer of the network from a model trained for ImageNet~\cite{imagenet} classification, and then train the network with a standard cross-entropy loss on the entire training set (except for the results on the black-box setting, as described later).

\subsection {Geolocation Performance of Images vs Collections}

We begin with experiments that demonstrate that the threat to location
privacy is significantly heightened when aggregating information
across multiple images. We evaluate per-image location accuracy from
our trained classifier in Table~\ref{table:full_accuracy}, and compare
it to accuracies of predictions from different collections using
\eqref{eq:probmult} as a collection-based classifier. While
performance is relatively modest with single images, it increases
significantly when we use photo collections, even with only 16
images, yielding a significant privacy risk.

Note that accuracies for user designated albums are considerably lower
because images in these albums are often correlated with similar
content in many images, since most albums have images all taken by a single photographer and in similar environs. Nevertheless, some user albums may still pose a privacy risk, and our top-$k$ guarantee approach can be used on these, since it would delete photos only in albums that compromised location privacy.

\begin{table}[!t]
\centering
\begin{tabular}{lccc}  
\toprule
Collection Type   & top-1 &  top-5 & protected-$k$\\
\midrule
Single Image  & 0.20  & 0.38 & 59.58 \\
R-16        & 0.97  & 1.00 & 0.06\\
R-32        & 0.99  & 1.00 & 0.01 \\
R-64        & 1.00 & 1.00 & 0.00 \\
R-128      &  1.00 & 1.00 & 0.00 \\
Albums      &  0.20 & 0.39 & 53.51 \\
\bottomrule
\end{tabular}
\caption{Geolocation performance on single images vs image collections. We show top-1 and top-5  accuracies (fraction of instances when the true location is in the top 1 or 5 predictions), and average protected-$k$ (where rank of true location is $k+1$) for all cases. Note lower protected-k represents more accurate geolocation. Here R-N refers to randomly constructed collections of size N, while Albums corresponds to user designated photo albums.}
\label{table:full_accuracy}
\end{table}

\subsection{Top-$k$ Guarantee with Minimum Deletions}
We now consider the effectiveness of our approach, starting with the
problem in \eqref{eq:topkmind}, where the goal is to minimize the
number of deleted photos while ensuring that the true location is not in the top-$k$ most likely predictions. We present these results in Table~\ref{table:untargeted}, for $k=1$ and $5$.

We find that when using our approach, we typically need to delete only a modest $\sim$30\% of the photos on average to achieve a top-1 guarantee, and $< 50\%$ of photos for the stronger top-5 guarantee---note that these averages are computed only over collections where the original predictions (prior to any deletion) were already accurate in the top-1 and top-5 sense, respectively. We also see that the optimal solution from our approach always requires fewer deletions than greedy selection---usually by a significant margin.

\begin{table}[!t]
\centering
\begin{tabular}{lcccc}  
\toprule
Collection& \multicolumn{2}{c}{top-1} & \multicolumn{2}{c}{top-5}\\ 
Type  & Optimal & Greedy & Optimal & Greedy\\
\midrule
R-16          & 0.32 & 0.37 & 0.45 & 0.53\\
R-32          & 0.31 & 0.43 & 0.47 & 0.59\\
R-64          & 0.31 & 0.49 & 0.48 & 0.64\\
R-128         & 0.31 & 0.53 & 0.49 & 0.68\\
Albums         &0.21 & 0.33 & 0.31 & 0.39 \\
\bottomrule
\end{tabular}
\caption{Average fraction of deleted photos selected to achieve top-1 and top-5 geolocation privacy. We compute averages only over instances where predictions on the original whole collection was accurate in the top-1 / top-5 sense.}
\label{table:untargeted}
\end{table}

Thus, our method ensures location privacy while still allowing users to share a significant fraction of their images.

\subsection{Maximum protected-$k$ under Deletion Budget}

Next, we consider our approach for our second problem variant in \eqref{eq:maxk} where, given a fixed budget on the number of images to delete from each collection, the goal is to maximize the value of $k$ such that the true location is not in the top-$k$ predictions. Table~\ref{tab:maxk} reports results for this setting on our various collection types, from optimal selection with our approach as well as from the greedy selection method (this time, averaging over all collection instances).

We find that again, optimal selection with our approach yields
significantly better privacy protection than the greedy
baseline. Moreover, the same \emph{fraction} of deleted images
provides greater protection when the size of the collection is
small. Note that average protected-$k$ values are higher for the user
albums, because the corresponding values for the uncensored
collections were already high (Table~\ref{table:full_accuracy}), so
much so that it is unnecessary to consider deletion budget above 25\%,
with $k$ already above 100 at that point (hence the blank entries in
the table).

\begin{table}[!t]
\centering
\begin{tabular}{llccccc}
\toprule
  \multicolumn{1}{l}{{Coll. }} & & \multicolumn{4}{c}{Deletion Budget ($\%$ Photos)} \\

  {Type} & Method & ${12.5\%}$ & ${25\%}$ & ${37.5\%}$ & ${50\%}$ \\ 

  \midrule
  \multirow{2}{*}{R-16}    & Optimal &  0.67 & 3.00 & 10.20 & 29.12 \\
      & Greedy  & 0.46 & 1.88 & 6.22 & 17.67\\
  \midrule
  \multirow{2}{*}{R-32}    & Optimal &  0.17 & 1.14 & 4.84& 17.49\\
      & Greedy  & 0.07 & 0.42 & 1.90& 7.15 \\
  \midrule
  \multirow{2}{*}{R-64}   & Optimal &  0.05 & 0.57 & 2.87 & 10.20\\
     & Greedy  & 0.02 & 0.09 & 0.57 & 2.80   \\
  \midrule
  \multirow{2}{*}{R-128}  & Optimal &  0.03 & 0.38 & 2.06 & 8.07\\
    & Greedy & 0.01 & 0.03 & 0.21 & 1.29  \\
  \midrule
  \multirow{2}{*}{Albums}   & Optimal & 77.09 & 103.18 & - &-\\
     & Greedy & 70.78 & 86.31 & - & - \\
\bottomrule
\end{tabular}
\caption{Average protected-$k$, such that the true label is not in top-$k$, for different deletion budgets as percentages of collection size. Note higher protected-$k$ implies greater privacy.}
\label{tab:maxk}
\end{table}

\begin{table}[!t]
\centering
\begin{tabular}{lcc}

\toprule
  Coll. & \multicolumn{2}{c}{{Deletion Budgets}} \\

  Type & {12.5\%$\rightarrow$25\%} & {25\%$\rightarrow$50\%} \\ 

\midrule
  R-16      & 0.27 & 0.10   \\

  \midrule
  R-32    & 0.32 & 0.15  \\

  \midrule
  R-64   & 0.34 & 0.19  \\

  \midrule
  R-128  & 0.33 & 0.22  \\
\bottomrule
\end{tabular}
\caption{Fraction of images that were deleted with a lower-deletion budget, but NOT with a higher one.}
\label{tab:maxk2}
\end{table}

In addition, we check if the photos deleted for
a given collection for one deletion budget, are also deleted given a
larger  budget. As our results in Table~\ref{tab:maxk2} indicate, this
is not the case: when more deletions are possible, the optimal
solution may involve keeping some of the photos that would have been
deleted under a lower budget. This indicates that any sequential greedy strategy is likely sub-optimal.

\begin{figure*}[!t]
 \centering
  \begin{tabular}{ccccc}
    \rotatebox{90}{~~~~~~~~ ~~~~~~Top-1}
    & \includegraphics[width=0.45\columnwidth]{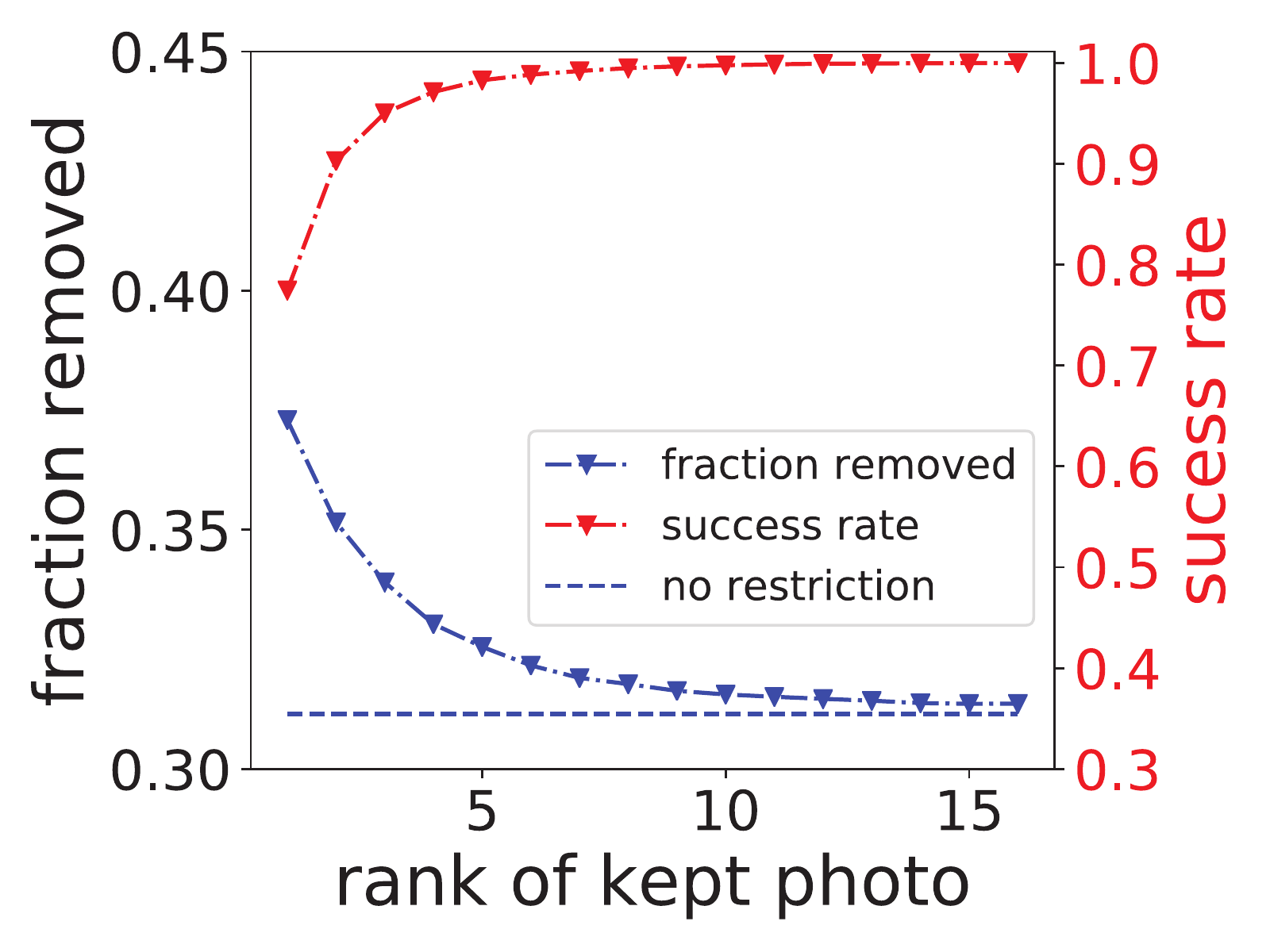}
    & \includegraphics[width=0.45\columnwidth]{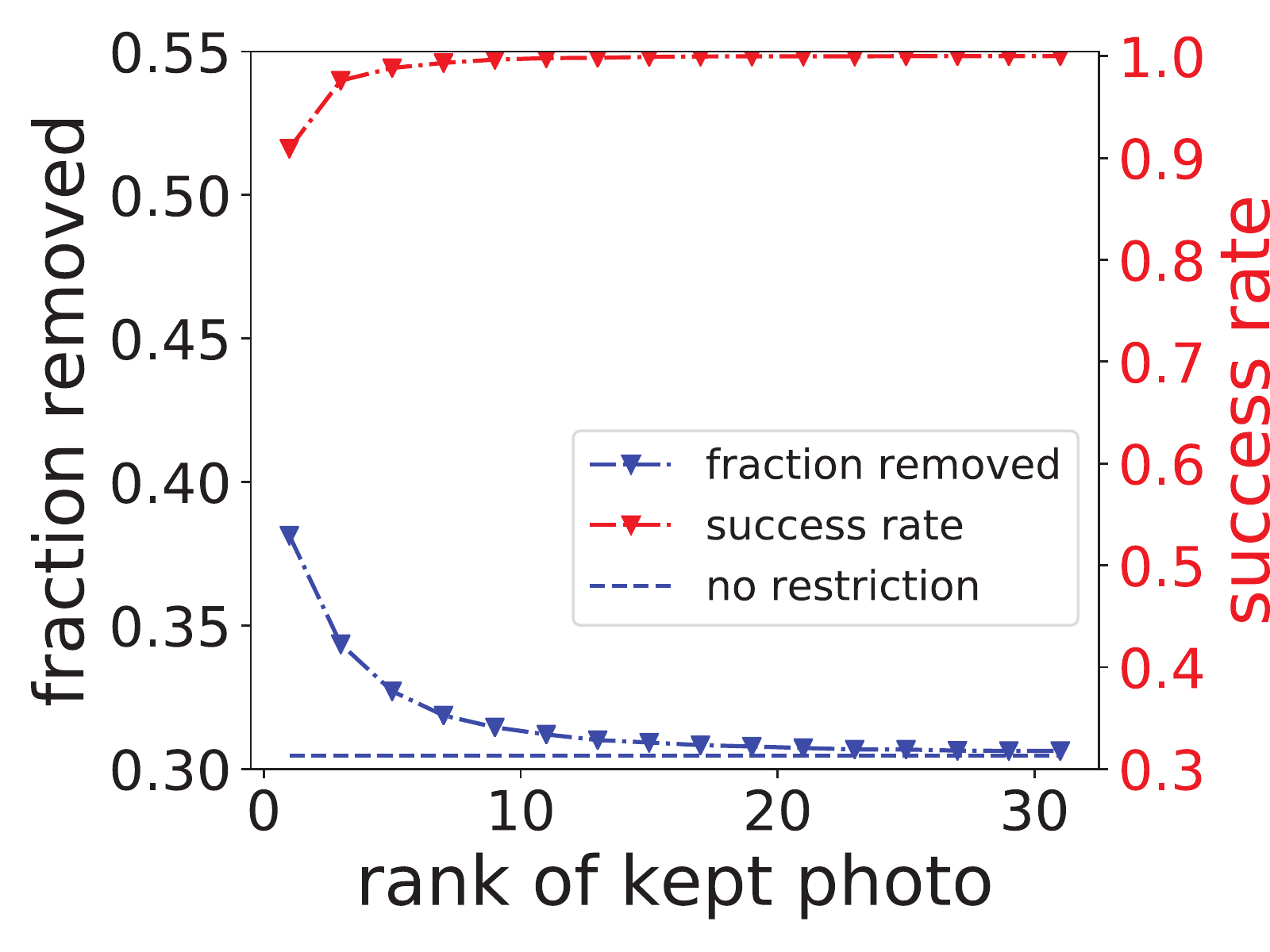}
    & \includegraphics[width=0.45\columnwidth]{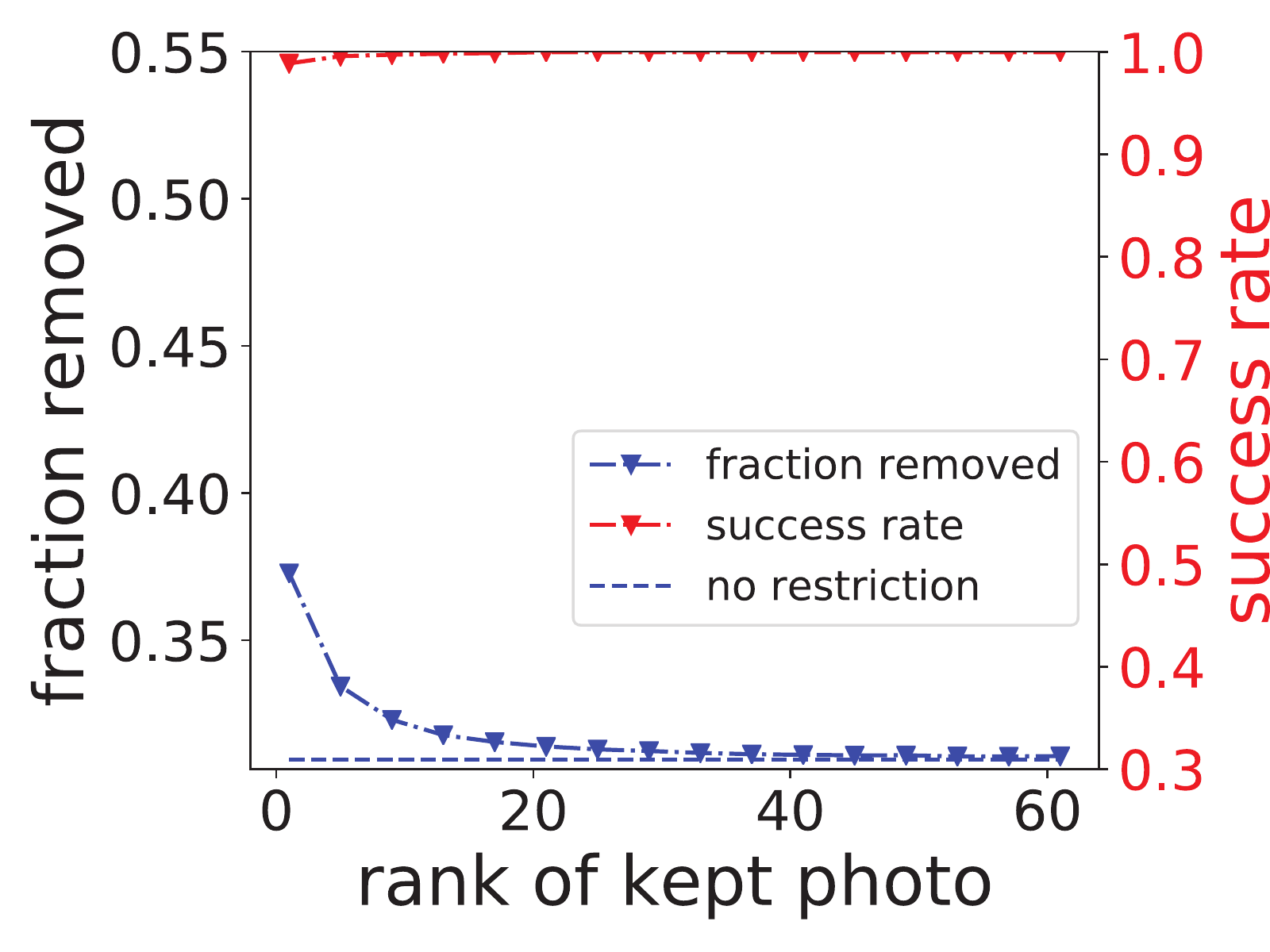}
    & \includegraphics[width=0.45\columnwidth]{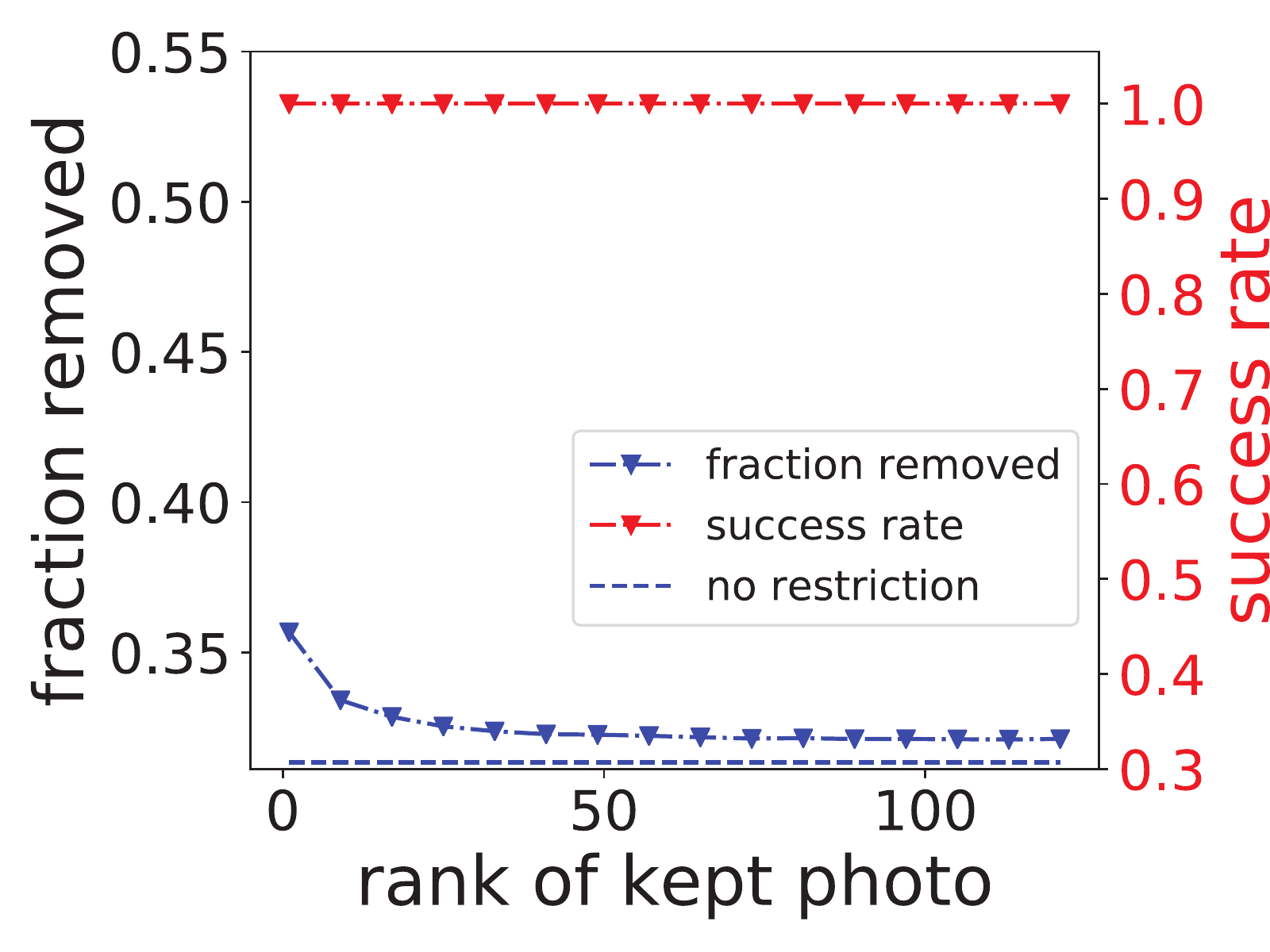}\\
    \rotatebox{90}{~~~~~~~~ ~~~~~~Top-5}
    & \includegraphics[width=0.45\columnwidth]{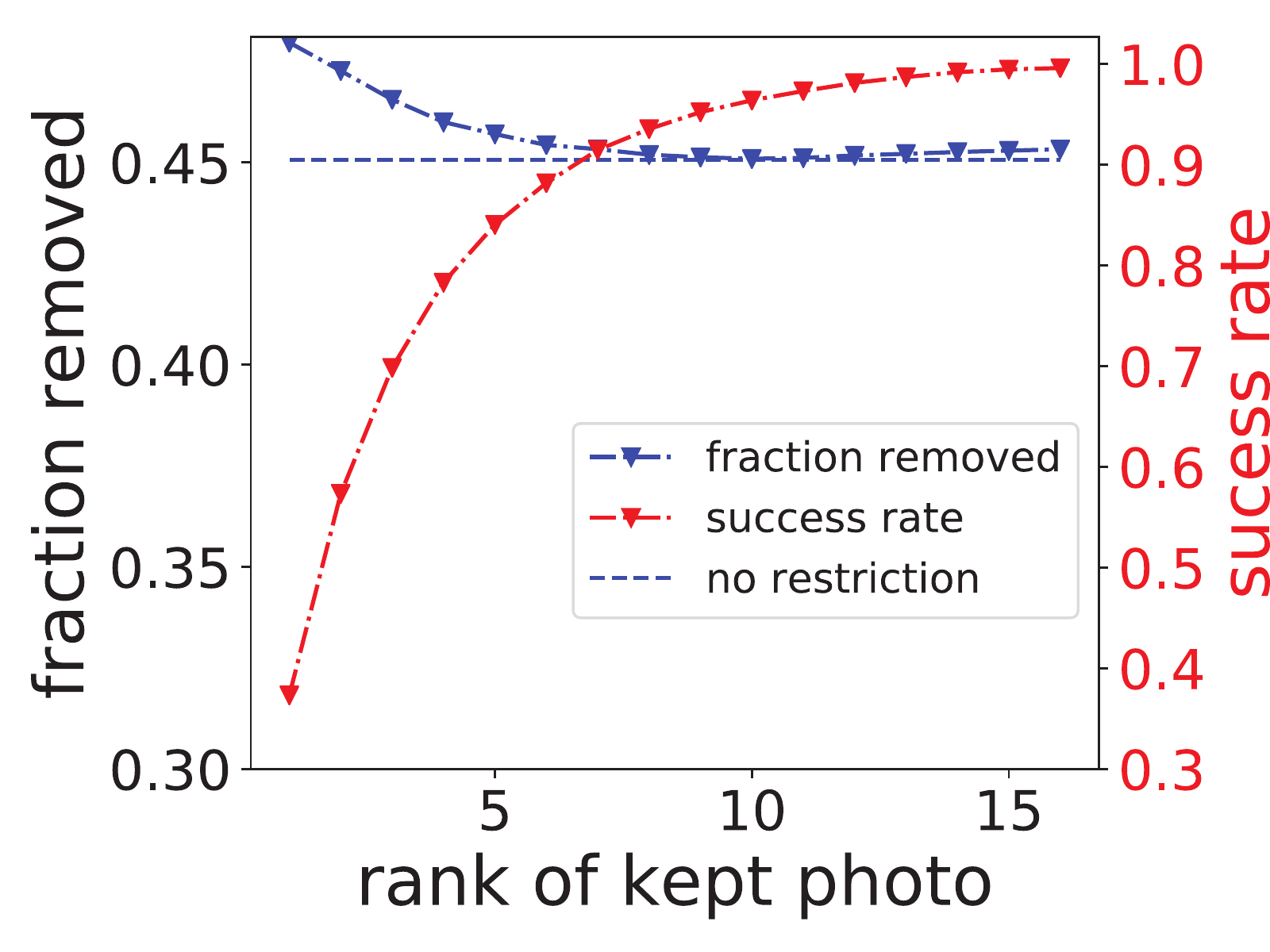}
    & \includegraphics[width=0.45\columnwidth]{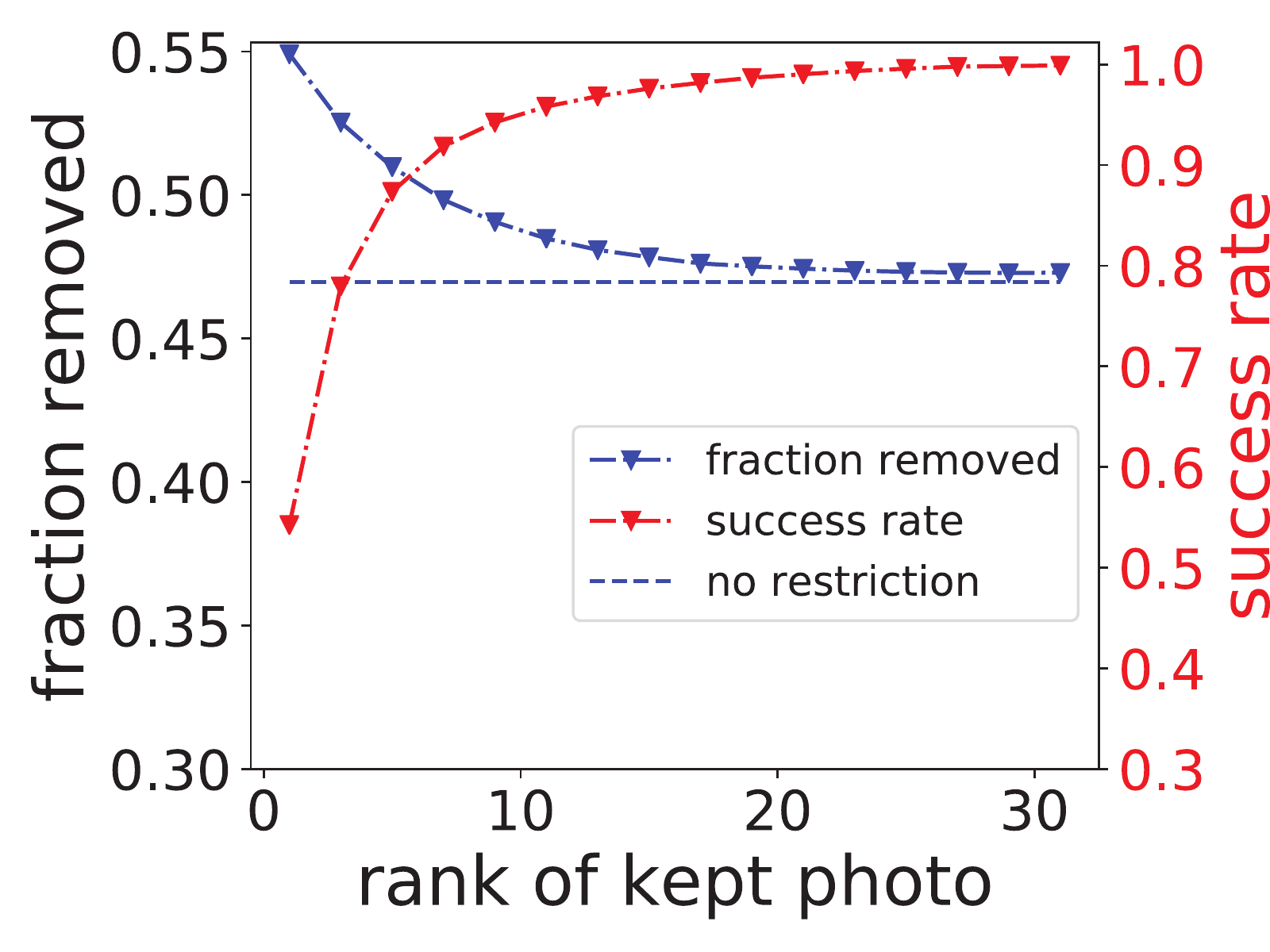}
    & \includegraphics[width=0.45\columnwidth]{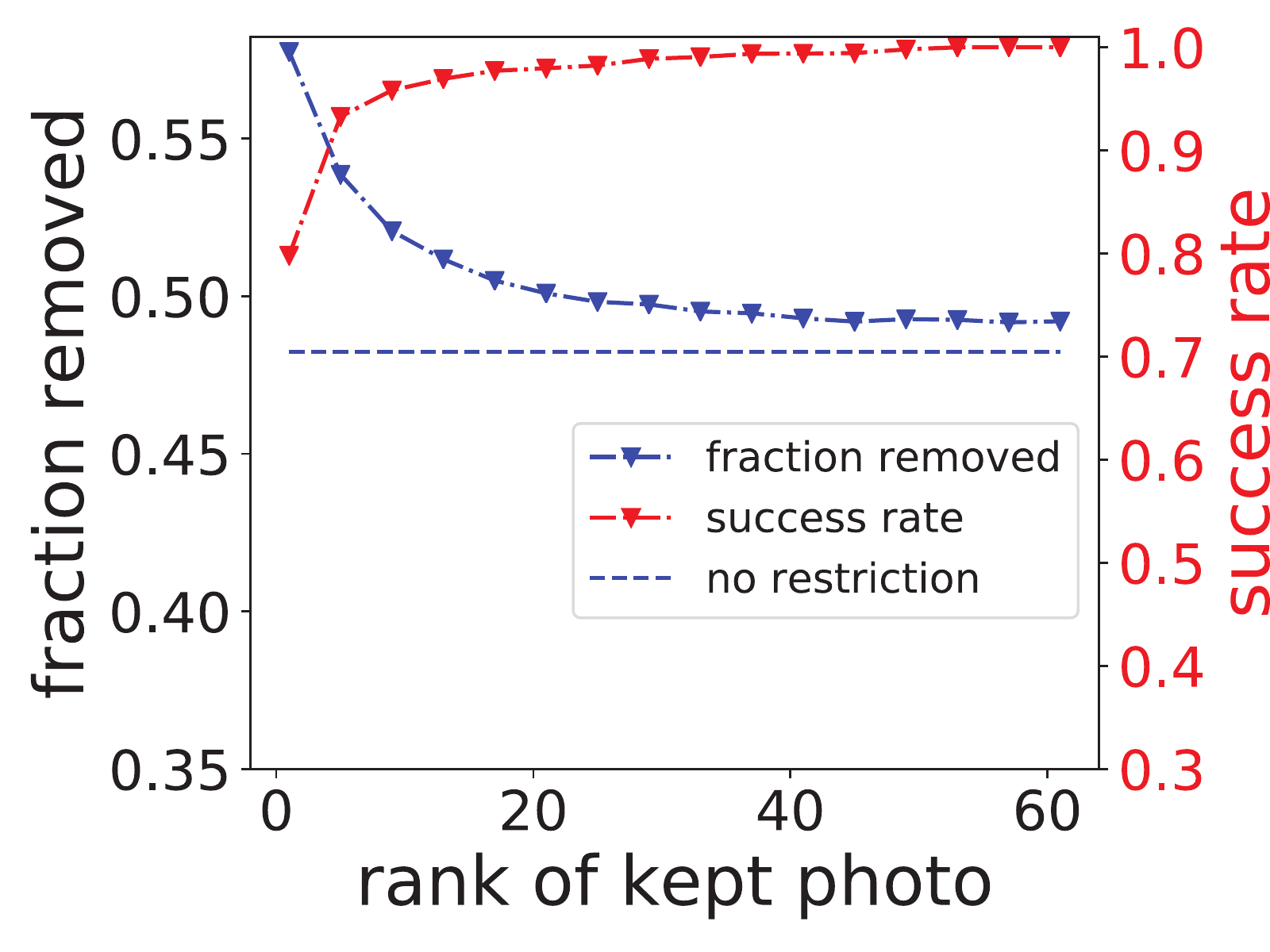}
    & \includegraphics[width=0.45\columnwidth]{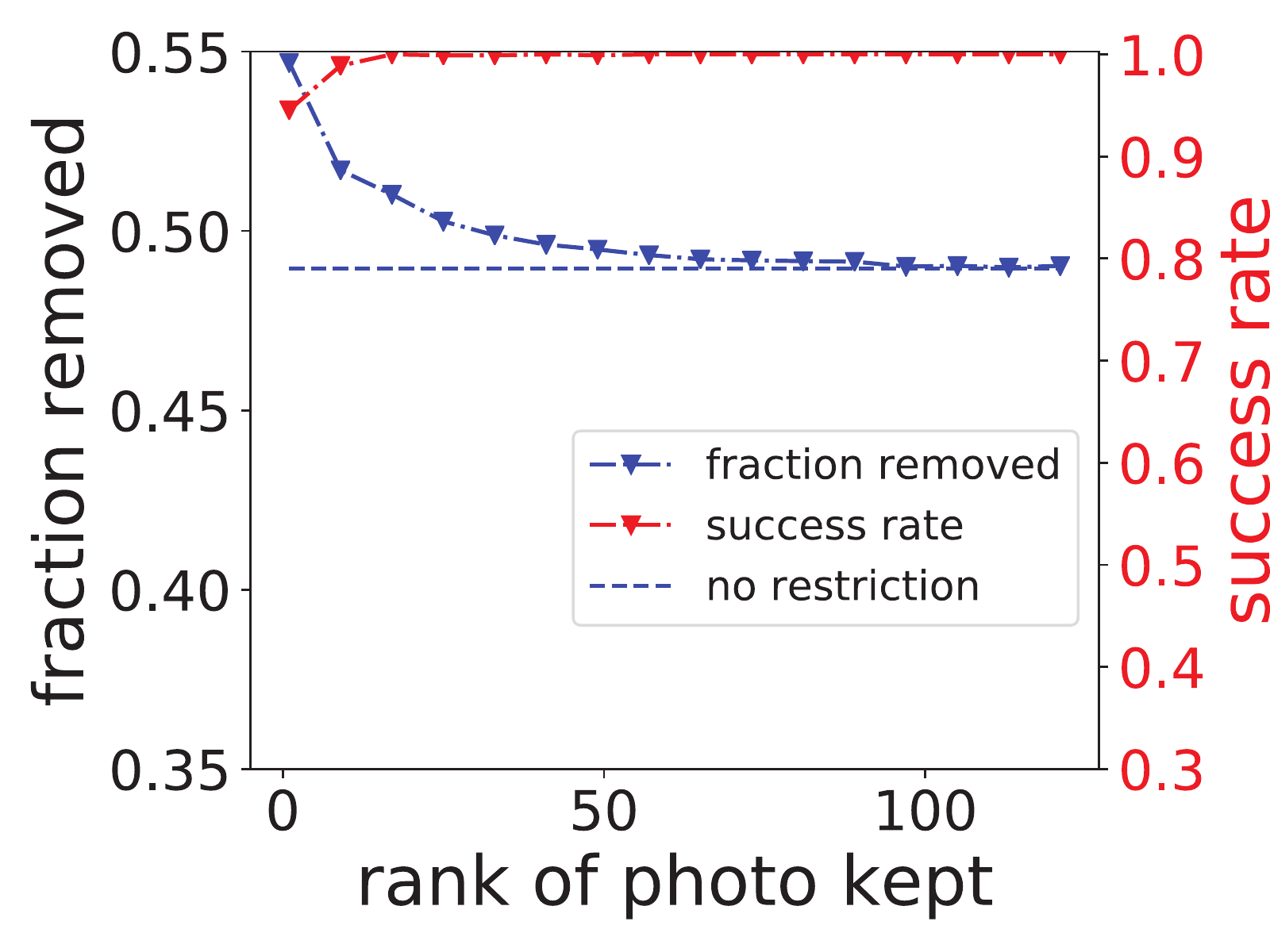}\\
    & \small R-16 & \small R-32 & \small R-64 & \small R-128
  \end{tabular}
  \caption{Effect of user-specified constraint to retain a chosen photo, while providing top-$k$ guarantee. Rank of the chosen photo is with respect to its score for the true location $S^t_i$, success rate indicates fraction of collections for which a feasible deletion set was found, and fraction of photos removed is reported averaged over successful cases.}
  \label{fig:users}
\end{figure*}

\subsection{User Preference}

Next we consider the case where we incorporate user-specified constraints to retain certain photos in each collection. Naturally, the effectiveness of privacy protection will depend on how informative that photo is towards location: preserving privacy may be impossible without deleting images with clearly recognizable landmarks. For a systematic evaluation, we always choose one photo to retain per collection, with different experiments where we consistently choose the photo with the same ``rank'' in terms of true location scores (rank 1 implying photo with highest $S_i^t$).

Figure~\ref{fig:users} presents these results for different collection
types, for the tasks of top-$1$ and -$5$ guarantees. We show the
success rate (fraction of collections where it was possible to achieve
a guarantee without deleting the specified photo) and the average fraction of deletions required (for successful cases). For comparison, we show the baseline fraction of deletions from Table~\ref{table:untargeted}. We see that putting constraints does degrade performance---requiring more deletions and sometimes leading the problem to be infeasible. This is more so for small collections. However, the rank of the photo plays a major role, and the performance drop is negligible when constraining photos that have higher rank (i.e., lower scores for the true location), especially for larger collections. Such settings are likely to match user expectations of maintaining privacy while retaining location neutral photographs.

\subsection{Black-Box Protection}

Finally, we evaluate the case when we do not have access to true classifier scores but only a proxy. For this case, we train two VGG network classifiers, each on two separate halves of the training set. We then use one classifier to derive our proxy scores and apply our approach for top-$k$ guarantees, and evaluate its success by  determining the fraction of collections for which this guarantee holds when using the other classifier. These results are shown in Figure~\ref{fig:margin}, where we report both success rate and the number of deletions required for different values of margin $\theta$. We see that our success at defeating an unknown classifier increases with higher margins, albeit at the cost of an increased number of deletions.

\begin{figure}[!t]
\centering
\begin{tabular}{cc}
\includegraphics[width=0.45\columnwidth]{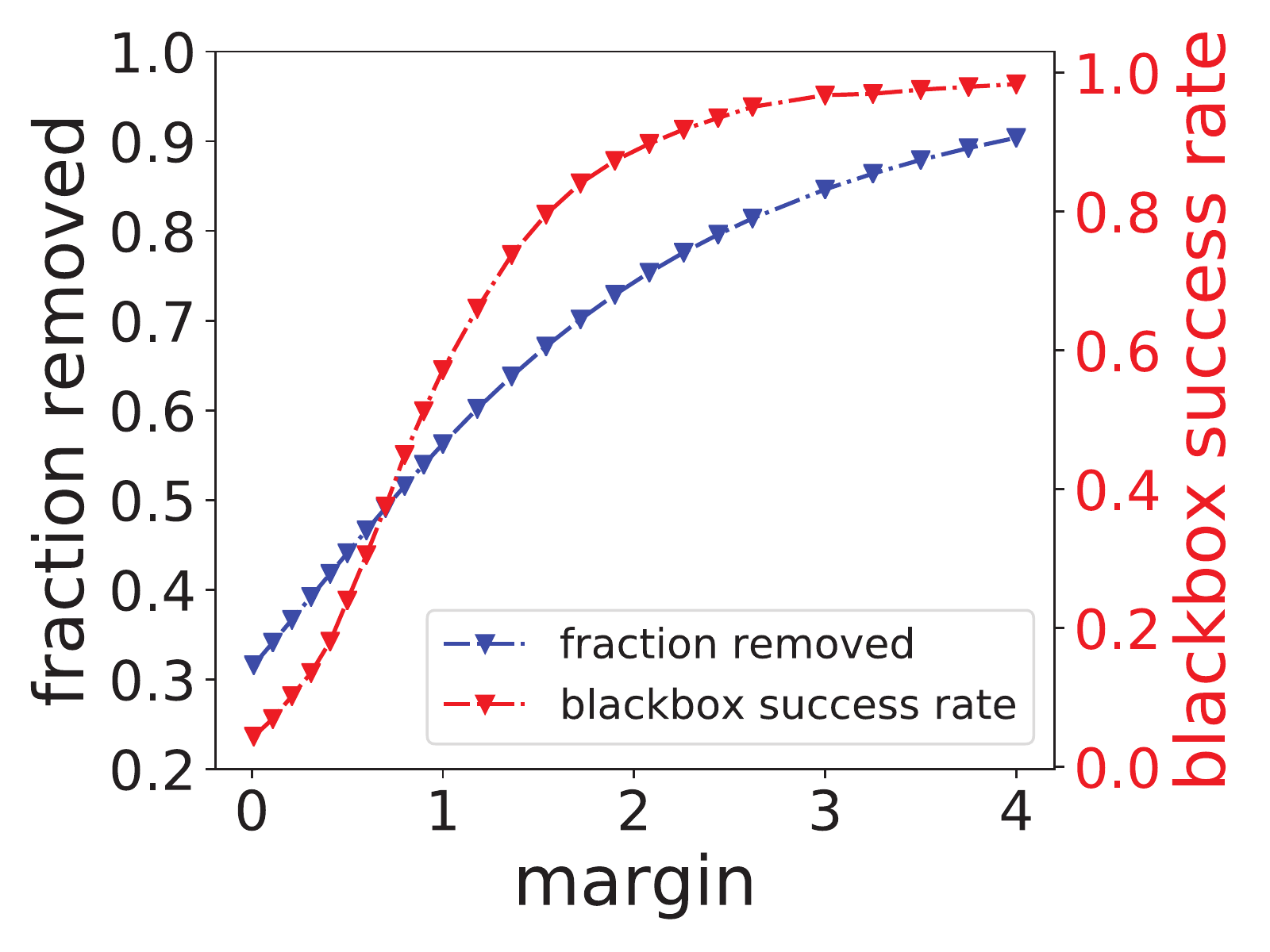} & 
\includegraphics[width=0.45\columnwidth]{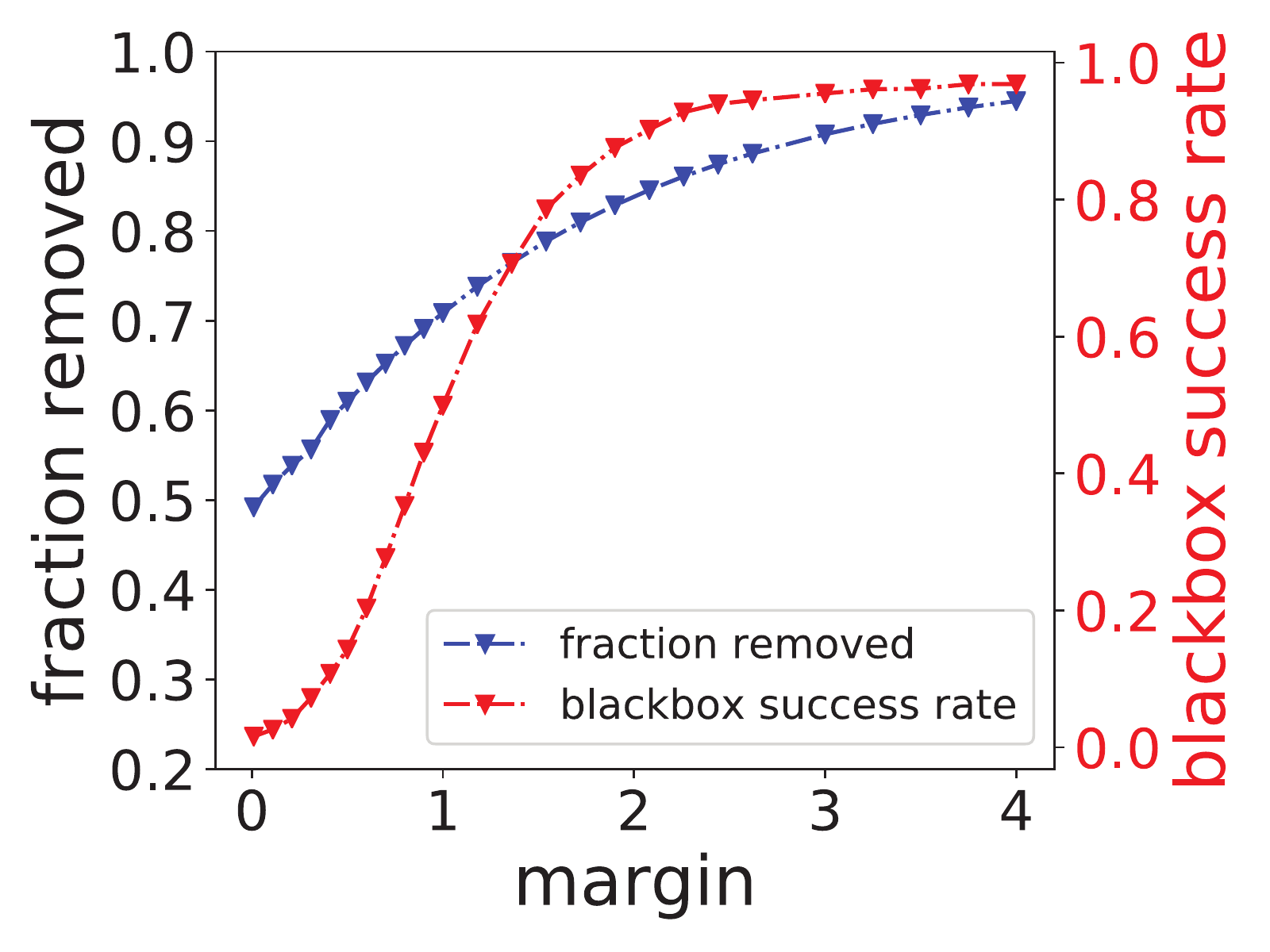} \\
\small{top-1} & \small{top-5}
\end{tabular}
\caption{Effectiveness of black-box attacks as a function of the
  margin $\theta$ for R-128 collections.}
\label{fig:margin}
\end{figure}

\subsection{Running Time}

Although our optimization problem is NP-hard, we find that a modern MILP
solver, CPLEX, returns solutions in reasonable time: $~\sim$9.2s for a R-128 collection and $\sim$4.8s for R-64, for the top-5 guarantee setting.

\section{Conclusion}
Our work demonstrated that photo \emph{collections} pose a clear risk to location privacy---simply from aggregating the outputs of existing deep neural network-based location classifiers---and proposed a framework for ameliorating this risk, through an algorithm for selecting optimal photos for deletion. While we considered that entire set was available to us prior to selection, an interesting direction of future work relates to \emph{online} versions of the selection problem---choosing to delete or keep/post a photo without knowing what other photos may be added to the collection.

It is also worth highlighting the assumptions and limitations of our formulation. Our selection method requires access to at least a reasonable proxy for the model that an adversary may use for location inference---this is implicit in the choice of the margin $\theta$ in the black-box setting. We also assume there is no source of location information other than the images themselves, and accounting for side-channels requires extensions to our framework (which could potentially be modeled as a score offset provided by the side-channel). We also assume conditional independence between different images in a collection, in deriving the collection-level location score as the sum of per-image scores. However, it is conceivable that a more sophisticated location classifier could learn to model user behavior, and leverage higher-order statistics to find, for example, that the presence of certain combinations of images in a collection are more indicative of certain locations.

\section*{Acknowledgments}
This work was partially supported by the National Science Foundation
CAREER award (IIS-1905558) and NVIDIA.

\bibliographystyle{aaai}
\bibliography{paper}

\clearpage
\appendix
\section*{Supplement}

\section{Construction of Scores for Observation 1}
In Table~\ref{tab:album}, We show construct the construction of a collection with probabilities that lead greedy selection to be arbitrarily sub-optimal. 
\begin{table}[!h]
\centering
\begin{tabular}{lrrr}  

\toprule
photos          & $c_1$                        & $c_2$                        & $c_t$                     \\ \midrule
$X_1$       & $\frac{1}{3}$                & $\frac{1}{3}$                & $\frac{1}{3}$             \\ 
...             & ...                          & ...                          & ...                       \\ 
$X_{N}$       & $\frac{1}{3}$                & $\frac{1}{3}$                & $\frac{1}{3}$             \\
$X_{N+1}$ & $a$                            & $\frac{2}{3} + \epsilon - a$ & $\frac{1}{3} - \epsilon$  \\ 
$X_{N+2}$ & $\frac{2}{3} + \epsilon - a$ & $a$                            & $\frac{1}{3} - \epsilon$  \\ 
$X_{N+3}$ & $\frac{1}{3} + \epsilon$     & $\frac{1}{3} + \epsilon$     & $\frac{1}{3} - 2\epsilon$ \\\bottomrule
\end{tabular}
\caption{Probabilities $p(c_j|X_i)$ for a collection where greedy deletion is arbitrarily sub-optimal.}
\label{tab:album}
\end{table}
Note that $a$ and $\epsilon$ must be such that $0 < a < \frac{2}{3} + \epsilon$ and $0<\epsilon<\frac{1}{6}$ to ensure all elements in the matrix are valid probabilities. Further, we must ensure that the most likely location prediction from the collection---without deletions---is the true location $c_t$. Accordingly, we set $a = \frac{(\frac{1}{3} - \epsilon)^{2}\times(\frac{1}{3} - 2\epsilon)}{2(\frac{1}{3} + \epsilon)^{2}}$. It is easy to verify that $a>0$. Since $\frac{(\frac{1}{3} - \epsilon)^{2}}{2(\frac{1}{3} + \epsilon)^{2}}<1$ and $\frac{1}{3} - 2 \epsilon < \frac{2}{3} + \epsilon$, we have $a < \frac{2}{3} + \epsilon$. Also, since  $ a > 0 $, we have $\frac{2}{3} + \epsilon -a < 2(\frac{1}{3}+\epsilon)$. Consequently,
\begin{align}
&  \exp(S^1) = \exp(S^2) \notag\\
	   &= \frac{(\frac{1}{3} - \epsilon)^{2}\times(\frac{1}{3} - 2\epsilon)}{2(\frac{1}{3} + \epsilon)^{2}}  \times  (\frac{2}{3} + \epsilon - a) \times (\frac{1}{3} + \epsilon) \times (\frac{1}{3})^{N}   \notag\\
            &= \frac{(\frac{1}{3} - \epsilon)^{2}\times(\frac{1}{3} - 2\epsilon)}{2(\frac{1}{3} + \epsilon)}  \times  (\frac{2}{3} + \epsilon - a)  \times (\frac{1}{3})^{N}  \notag\\
            &=[\frac{1}{2(\frac{1}{3} + \epsilon)} \times (\frac{2}{3} + \epsilon - a)] \times [( \frac{1}{3} - \epsilon)^{2}\times(\frac{1}{3} - 2\epsilon)  \times (\frac{1}{3})^{N}]  \notag\\
            &< [\frac{1}{2(\frac{1}{3} + \epsilon)} \times 2(\frac{1}{3}+\epsilon)] \times [( \frac{1}{3} - \epsilon)^{2}\times(\frac{1}{3} - 2\epsilon)  \times (\frac{1}{3})^{N}] \notag\\
            &= \exp(S^t).
\label{exp:constraint1}
\end{align}

Note that a greedy algorithm applied to these probabilities will first select $X_1 \ldots X_N$ for deletion before the remaining images (because they have the highest values of $S^t_i$). But deleting these images will make no difference to the rank of the true location $t$ in the corresponding score vector $S'$.

Conversely, optimal selection need only delete $X_{N+1}$ and $X_{N+2}$ to ensure ${S'}^t < {S'}^1$ and ${S'}^t < {S'}^2$:
\begin{align}
  {S'}^1 = {S'}^2 = \log (\frac{1}{3} -2 \epsilon) + N \log (\frac{1}{3})\notag\\
  < \log (\frac{1}{3} + \epsilon) + N \log (\frac{1}{3}) = {S'}^t.
\end{align}

\newpage

\section{Algorithm for $k=1$}

\begin{algorithm}[h]
\caption{Untargeted with k = 1}
\label{alg:algorithm}
\textbf{Input}: $c_t$, $\mathcal{C}^{-c_t} = = C \setminus c^{t}$, $\mathcal{A}$ \\
\textbf{Output}: $\mathcal{D}$ 
\begin{algorithmic}[1] 
\STATE Initialize $\mathcal{D} = \emptyset$, d = 1
\WHILE{$d < |\mathcal{A}|$}
\FOR {$c_j \in \mathcal{C}^{-c_t}$}
\STATE Sort $S^{j} - S^{t}$ in ascending order. 
\STATE Delete the first $d$ elements in $S^{j} - S^{t}$ and update $\mathcal{D}$ with indexes of the deleted elements. 
\IF {$\sum_{i}(S_i^{j} - S_i^{t}) > 0$}
\STATE \textbf{return} $D$. 
\ENDIF
\ENDFOR
\ENDWHILE
\STATE \textbf{return} No approach to protect. 
\end{algorithmic}
\end{algorithm}

\section{Linearizing Constraint~\eqref{eq:topkilp}}
Since we need all constraints and objectives to be linear to use an MILP solver, we adopt a standard linearization technique to arrive at a modified set of constraints. Choosing $T$ to be a sufficient large number, and $w_j = h_j v_j$, the following optimization problem with only linear constraints is equivalent to ours (for the top-$k$ guarantee):
\begin{align}
    &\text{max} \qquad \sum_{i} z_i &\notag \\
    &\text{s.t.} \qquad w_j \geq 0, \quad j=1,2,...,M-1 &\notag \\
    &    \qquad    \quad  w_j \geq -T h_j &\notag\\
    &    \qquad    \quad  w_j \leq T h_j &\notag\\
    &    \qquad    \quad  w_j \geq v_j - T(1-h_j) &\notag\\
    &    \qquad     \quad  w_j \leq v_j + T(1-h_h) &\notag\\
    &     \qquad     \quad   \sum_{j} h_j \geq k &\notag \\
    &\text{where} \quad v_j = \sum_i z_i \left( S_i^j - S_i^t \right) & 
\end{align}
A similar approach can be used for the fixed budget problem.





\end{document}